\documentclass{article}

\usepackage[final,sglblindworkshop]{neurips_2026}
\workshoptitle{Neural Network Artifacts as a New Data Modality}

\usepackage[utf8]{inputenc}
\usepackage[T1]{fontenc}
\usepackage{url}
\usepackage{booktabs}
\usepackage{graphicx}
\usepackage{amsmath}
\usepackage{amssymb}
\usepackage{nicefrac}
\usepackage{microtype}
\usepackage{xcolor}
\usepackage{hyperref}
\hypersetup{hidelinks}

\title{Auditing Finite-Lag Statistics of Recurrent Representations}
\author{Kanishka Reddy \\
Department of Applied Mathematics \\
University of Washington \\
Seattle, WA 98105 \\
\texttt{kani@uw.edu}}

\begin{document}
\maketitle

\begin{abstract}
Hidden-state trajectories are computational traces that record how recurrent representations evolve as networks process sequences. When exact state repeats are unavailable, one way to estimate conditional motion is to borrow successors from nearby states, yielding a finite-bandwidth successor smoother. However, this procedure can change the displacement being measured. At a median-distance bandwidth, write-phase displacement energy is about $1.88\times$ the exact value for GRU and LSTM. Elman instead shows near agreement in total energy despite large, opposing finite-bandwidth contributions. We therefore compute the conditional-mean energy fraction exactly on repeat-copy. Training reverses the architecture ordering of this fraction, with GRU and LSTM rising to about $94\%$ while Elman falls to about $55\%$. The centered remainder is small in energy but can still expose input information. For LSTM, linear four-bit decoding rises from $0.804$ on total displacement to $1.000$ after exact conditional centering.
\end{abstract}

\section{Introduction}
\label{sec:intro}

Recurrent networks trace a path through hidden-state space as they process a sequence. We ask how much a state tends to move over a fixed time lag. When the same state is rarely observed twice, a simple and appealing approach is to borrow successors from nearby states and average their motion using standard kernel-regression weights. But a borrowed successor began at a different source. Measuring from the query state adds that source difference to the actual movement.

Existing work compares learned representations through similarity and intrinsic dimension \citep{raghu2017svcca,kornblith2019similarity,ansuini2019intrinsic,yeom2025internal} and studies trained recurrent networks as dynamical systems \citep{sussillo2013opening,maheswaranathan2019reverse,driscoll2024flexible}. Lagged operators provide another view of temporal evolution \citep{schmid2010dynamic,williams2015koopman,perezhernandez2013tica,wu2020vamp,mardt2018vampnets}. Checkpoints and computational traces let us test whether a temporal summary is measuring what it claims to measure. Exact interventions provide a direct reference when they are available, while held-out prediction gives a practical check when they are not. We use both to ask whether finite-bandwidth smoothing preserves the motion it is meant to summarize.

\paragraph{Contributions.}
We first evaluate a Gaussian successor smoother and find that dense neighborhoods can distort pooled displacement energy, while Elman shows that aggregate agreement can coexist with large opposing finite-bandwidth contributions. Next, by examining every possible next write input on repeat-copy, we find that training reverses the architecture ordering of the conditional-mean energy fraction. Finally, we ask whether the low-energy centered remainder also contains little input information and find that exact conditional centering makes the incoming four-bit pattern perfectly linearly decodable.

\section{Finite-Lag Motion}
\label{sec:statistics}

\paragraph{Measured representation and conditioning state.}
We measure motion in the hidden representation $Z_t=h_t$. The conditioning state $S_t$ is the information used to predict that motion. For Elman and GRU, $S_t=h_t$. For LSTM, the complete recurrent state is $(h_t,c_t)$. Let $D=Z_{t+\Delta}-Z_t$ be the displacement over $\Delta$ updates, with elapsed time $\tau=\Delta\,dt$. Its conditional moments are
\begin{align}
m_\Delta(s) &= \tau^{-1}\mathbb E[D\mid S_t=s], &
C_\Delta(s) &= \tau^{-1}\operatorname{Cov}(D\mid S_t=s), \\
G_\Delta(s) &= \frac{1}{2\tau}\mathbb E[DD^\top\mid S_t=s].
\end{align}
Here $m_\Delta$ describes average motion, $C_\Delta$ describes variation around that average, and $G_\Delta$ combines both.

\paragraph{Conditional-mean and residual motion.}
Conditional mean motion is the average displacement expected from a given state. Residual motion is what remains after subtracting that average. The second moment separates these two contributions through
\begin{equation}
2G_\Delta(s)=C_\Delta(s)+\tau m_\Delta(s)m_\Delta(s)^\top.
\label{eq:decomp}
\end{equation}
The fraction of total displacement energy in the conditional mean is
\begin{equation}
F(S)=\frac{\mathbb E\|\mathbb E[D\mid S]\|^2}{\mathbb E\|D\|^2}.
\label{eq:f}
\end{equation}
This fraction counts shared drift even when it is the same at every state, so it is not a variance-explained $R^2$. At the population level, adding conditioning information cannot reduce the fraction. If $S_1$ is determined by $S_2$, then $F(S_2)\geq F(S_1)$.

\paragraph{Pooled displacement energy.}
Pooled displacement energy averages squared displacement lengths over source states. If $\rho$ describes their distribution and $\bar G_\Delta^\rho=\int G_\Delta(s)\rho(ds)$, then
\begin{equation}
2\tau\bar G_\Delta^\rho=\mathbb E[DD^\top],
\qquad
2\tau\operatorname{tr}(\bar G_\Delta^\rho)=\mathbb E\|D\|^2.
\label{eq:pooled}
\end{equation}
Thus the pooled trace measures overall movement, up to the factor $2\tau$. It does not separate motion predictable from the state from residual motion.

\section{What Successor Smoothing Measures}
\label{sec:estimation}

Repeat-copy provides an exact reference. At each sampled internal write state, we try all $16$ equally likely next inputs and average their displacements (Section~\ref{sec:experiment}). This gives the exact conditional mean and displacement energy at that state. The source states themselves are still sampled.

\paragraph{Successor smoothing.}
Suppose exact repeats of a source state are unavailable. In the smoother audited here, we condition on the measured hidden representation, $S=Z=h$. Write $X_i$ for an observed source and $Y_i$ for its successor. To estimate motion from a query state $X_q$, we borrow successors from nearby sources using Gaussian weights,
\begin{equation}
w_{qi}\propto\exp\!\left(-\frac{\|X_q-X_i\|^2}{4\varepsilon}\right),
\end{equation}
a standard kernel-regression construction also related to diffusion-based state-space affinities \citep{nadaraya1964estimating,watson1964smooth,coifman2006diffusion}. Larger bandwidth $\varepsilon$ gives more weight to distant sources. Successor smoothing treats each borrowed $Y_i$ as a possible successor of the query. It therefore measures $Y_i-X_q$, although the recorded movement was $Y_i-X_i$.

\paragraph{Where the mismatch enters.}
We call the difference between the neighbor's source and the query the source offset. The borrowed displacement therefore satisfies
\begin{equation}
Y_i-X_q=(Y_i-X_i)+(X_i-X_q),
\label{eq:mismatch-vector}
\end{equation}
so their squared lengths satisfy
\begin{equation}
\|Y_i-X_q\|^2
=\|Y_i-X_i\|^2+\|X_i-X_q\|^2
+2(Y_i-X_i)^\top(X_i-X_q).
\label{eq:mismatch-energy}
\end{equation}
The smoothed energy combines the neighbor's actual movement with a positive contribution from the source offset and a cross term that can add or subtract energy. This averages squared lengths, rather than squaring the average displacement. The source offset also changes the estimated mean. If
\begin{align}
\hat\mu_{\mathrm{succ}}(x)&=\sum_i w_iY_i-x, &
\hat\mu_{\mathrm{disp}}(x)&=\sum_i w_i(Y_i-X_i),
\end{align}
then
\begin{equation}
\hat\mu_{\mathrm{succ}}(x)-\hat\mu_{\mathrm{disp}}(x)
=\sum_iw_i(X_i-x).
\label{eq:mean-offset}
\end{equation}
The offset tends to zero as sources approach the query, but it can remain at finite bandwidth. Wider neighborhoods can therefore reduce sampling noise while changing the displacement being averaged.

\paragraph{Estimator audit.}
We measure these terms on solved repeat-copy models in Figure~\ref{fig:estimator}. At the median-distance bandwidth \citep{garreau2018median}, dense displacement energy is $1.88$ times the exact write-phase energy for both GRU and LSTM. Notably, Elman looks almost correct in aggregate even though its positive source-mismatch contribution is $0.75$ times the exact energy and is almost opposed by the negative cross term. Table~\ref{tab:mismatch-main} gives the raw decomposition, and Figure~\ref{fig:estimator}B shows that the mismatch grows as neighborhoods widen.

\begin{table}[t]
\caption{Dense write-phase energy decomposition at the original bandwidth. Values are means over five seeds.}
\label{tab:mismatch-main}
\centering
\small
\begin{tabular}{lrrrrr}
\toprule
architecture & exact & dense & neighbor & mismatch & cross \\
\midrule
Elman & 1.513 & 1.523 & 1.511 & 1.134 & $-1.121$ \\
GRU & 0.242 & 0.455 & 0.227 & 0.350 & $-0.121$ \\
LSTM & 0.300 & 0.564 & 0.283 & 0.452 & $-0.172$ \\
\bottomrule
\end{tabular}
\end{table}

\begin{figure}[t]
\centering
\includegraphics[width=\linewidth]{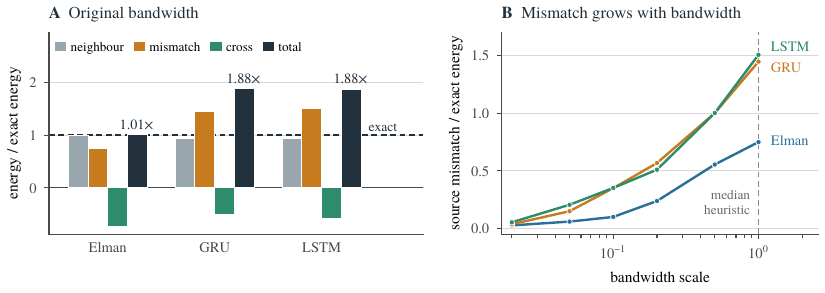}
\caption{Finite-bandwidth estimator audit. Panel A decomposes original-bandwidth dense energy after normalization by exact intervention energy. The dark bar is the signed sum. Panel B shows how source mismatch grows with bandwidth. Values are means over five seeds.}
\label{fig:estimator}
\end{figure}

Prediction favors smaller neighborhoods. In the original $h$-only sweep, successor smoothing selects scales $0.2$, $0.05$, and $0.02$ for Elman, GRU, and LSTM. Displacement smoothing selects $0.05$, $0.02$, and $0.02$. At the successor scales, the mean dense-to-exact energy ratios fall to $1.004$, $1.059$, and $1.025$. Tuning removes much of the total-energy inflation, but these full-sample ratios do not measure conditional-mean error on held-out sequences (Appendix~\ref{app:estimator}). Wider successor neighborhoods may reduce sampling noise or let the source offset cancel other estimation errors. These are possible explanations, not established mechanisms.

\paragraph{Takeaway.}
Dense successor smoothing can look accurate in aggregate while containing large finite-bandwidth contributions that are absent from direct displacement smoothing. Put differently, agreement in total energy alone does not certify conditional accuracy. 

What changes when we remove successor smoothing entirely?

\section{Exact Motion on Repeat-Copy}
\label{sec:experiment}

\paragraph{Our setup.}
Repeat-copy makes the question testable because every next write input comes from a small known alphabet. We train 64-unit Elman, GRU, and LSTM networks on four binary features with copy length $L=8$ and a four-step blank delay. We use $2048$ training sequences and $192$ validation sequences, retaining five solved seeds per architecture. The minimum held-out recall sign accuracy rounds to $0.9997$. Appendix~\ref{app:details} gives the optimization and intervention details.

Figure~\ref{fig:exact}A first gives the phase context. Recall produces the largest raw lag-one motion in every architecture, while the write phase supplies the controlled setting we need. At each internal write state, the next input is an independent random four-bit vector. We therefore collect $1344$ source states per seed and enumerate all $16$ possible next inputs instead of borrowing successors from neighboring states.

\paragraph{Exact conditional-mean motion.}
Let $u$ denote the next write input and $D(s,u)$ the exact one-step hidden displacement from recurrent state $s$. Enumerating every $u$ gives the conditional mean $\mu(s)=\mathbb E_u[D(s,u)\mid s]$ and therefore the fraction $F(S)$ in Equation~\ref{eq:f} without successor smoothing. Its complement is the residual energy fraction under the specified next-input distribution. For Elman and GRU, $S=h$. For LSTM, exact intervention uses the complete state $S=(h,c)$ while displacement is measured in $h$. 

\begin{figure}[t]
\centering
\includegraphics[width=\linewidth]{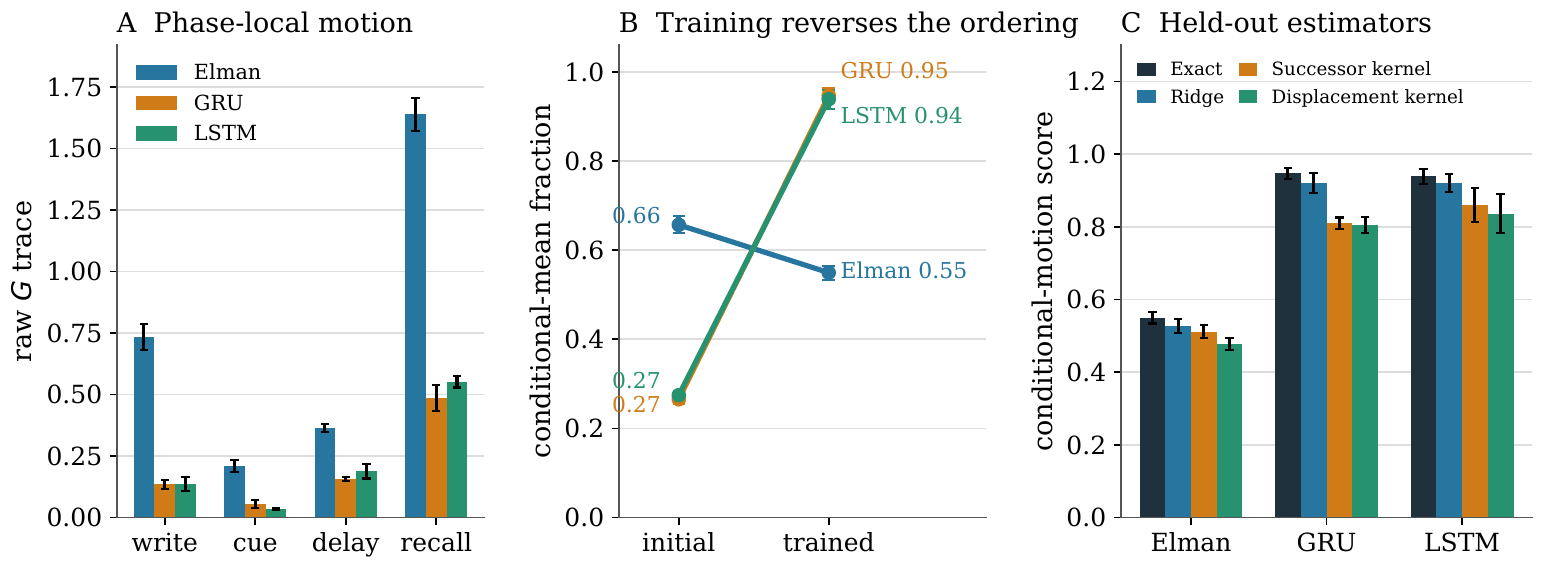}
\caption{Phase context, conditional-mean energy, and held-out prediction. Panel A shows raw lag-one movement across task phases. Panel B compares the exact conditional-mean fraction at initialization and after training. Panel C compares the trained exact fraction with sequence-grouped held-out prediction scores. Kernel bandwidths use nested selection. Error bars show one sample standard deviation over five seeds. All Panel C methods condition on $h$ for Elman/GRU and $(h,c)$ for LSTM.}
\label{fig:exact}
\end{figure}

\paragraph{Training reverses the ordering.}
We reconstruct each run at initialization with the deterministic procedure used by the training script and drive it with the same validation inputs. Figure~\ref{fig:exact}B shows a clear reversal in $F(S)$. Elman begins with the largest conditional-mean fraction at $0.657\pm0.019$, while GRU and LSTM begin near $0.27$. After training, GRU and LSTM rise above $0.93$, whereas Elman falls to $0.549\pm0.015$ and becomes the smallest. Interestingly, the large change is in the gated networks. Training shifts most of their write-motion energy into the component predictable from the specified recurrent state, while Elman moves more modestly in the opposite direction.

\paragraph{Takeaway.}
The trained ordering is not an architectural default present before learning. In other words, training changes how much write-phase displacement energy lies in the conditional mean, most strongly in the gated architectures.

Only about five to six percent of trained gated write-motion energy remains outside the conditional mean. Does this low-energy remainder also contain little information about the incoming input?

\paragraph{Input information in the residual.}
We decode the four-bit input from the same enumerated transitions using a linear probe \citep{alain2017probes}. A sequence-grouped ridge decoder reaches vector accuracy $1.000$ for Elman, $0.995\pm0.006$ for GRU, and $0.804\pm0.013$ for LSTM on total displacement. After subtracting the exact conditional mean from each displacement, it reaches $1.000$ for every architecture and seed. Thus a low-energy residual can expose the full input code, especially in LSTM. This intervention-derived probe does not establish persistence through the delay or use by the trained readout \citep{elazar2021amnesic}. Appendix~\ref{app:controls} gives details.

\paragraph{Takeaway.}
Energy allocation and linear accessibility are different properties. A low-energy centered residual can still expose the entire incoming code to a linear probe.

\paragraph{Can the exact quantity be predicted?}
Exact enumeration is available here only because the input alphabet is small. We therefore fit sequence-grouped predictors of the conditional mean and score their out-of-fold displacement prediction, selecting kernel bandwidths within training folds. Figure~\ref{fig:exact}C gives ridge scores of $0.527$, $0.921$, and $0.920$, all within $0.03$ of the exact fraction. Tuned Gaussian predictors recover substantial structure but remain farther from exact. Successor smoothing has lower direct mean error than displacement smoothing here. All methods in this comparison use matched conditioning, including $(h,c)$ for LSTM. Appendix~\ref{app:estimator} defines the evaluation and reports direct errors against the enumerated mean.

\section{Discussion}
\label{sec:discussion}

\paragraph{What the audit reveals.}
The most striking case looks nearly correct in aggregate. Elman's dense energy is close to exact even though a source-mismatch contribution worth about three quarters of exact energy is nearly opposed by the cross term. The cancellation does not by itself imply conditional error, since large opposing terms can also occur when successor smoothing is exact. GRU and LSTM provide the demonstrated distortion at $1.88\times$ exact. Elman instead shows why total-energy agreement alone cannot certify conditional accuracy.

\paragraph{What exact intervention reveals.}
Exact intervention reveals a reversal in the conditional-mean energy fraction, driven mainly by the large GRU and LSTM changes. This fraction measures the uncentered conditional mean under the specified input distribution, not autonomous recurrent computation. Conditional centering makes all four input bits perfectly linearly decodable, but it uses intervention-derived state information and does not show that this information persists through the delay or is used by the trained readout.

The controlled task identifies the residual source because the one-step transition is deterministic given recurrent state and next input, while conditioning on state alone averages over possible inputs. When enumeration is unavailable, cross-fitted prediction tests how much of the conditional mean is reproducible on held-out trajectories. The conditioning state, displacement definition, and bandwidth are therefore part of the measurement itself rather than incidental implementation choices.

\paragraph{Practical takeaway and scope.}
Check bandwidth sensitivity and compare successor and displacement smoothing at matched neighborhoods. Total-energy agreement alone is insufficient. Select bandwidths on training sequences and evaluate on held-out sequences, using direct conditional-mean errors when interventions are available. Our evidence is limited to one synthetic task, lag one, 64-unit networks, and five solved seeds per architecture. The magnitude of distortion and the architecture ordering require validation beyond this setting.

\clearpage
\bibliographystyle{plainnat}
\bibliography{references}

\clearpage
\appendix

\section{Experimental Details and Exact Intervention}
\label{app:details}

\paragraph{Task and training.}
Repeat-copy uses four binary channels plus a delimiter (input dimension $5$), $L=8$, and a four-step blank delay. We train 64-unit tanh Elman, GRU, and one-layer LSTM models \citep{elman1990finding,cho2014learning,hochreiter1997long} with a linear four-dimensional readout on $2048$ sequences and validate on $192$, using batch size $64$, gradient clipping at one, and recall-weighted MSE with weight five during recall. Elman uses Adam at $10^{-3}$ for at most $180$ epochs, and GRU and LSTM use $3\times10^{-3}$ for at most $320$, with minimum epoch $100$ and patience $60$. The minimum recall sign accuracy across all fifteen runs is $0.999674$.

We center hidden trajectories and divide by their global RMS radius. This leaves $F(S)$, dense-to-exact ratios, and relative-bandwidth Gaussian weights unchanged and only fixes units for raw displacement quantities such as Figure~\ref{fig:exact}A. Phase summaries group transitions by the source phase, including boundaries. The exact write audit excludes the write-to-cue boundary.

\paragraph{Exact intervention.}
We analyze internal write times $t=0,\ldots,L-2$, giving $1344$ source states per seed, and enumerate all $16$ next four-bit inputs at each source ($\Delta=1$, $dt=1$). Elman and GRU condition on $h$, and for LSTM we reconstruct $c_t$ and condition on $(h,c)$ while measuring displacement in $h$. CPU replay closely matches the saved trajectories, with small backend-dependent differences. The observed input reproduces the locally replayed successor to numerical precision. Averaging over sources gives Equation~\ref{eq:f}. Observed displacement energy agrees with the intervention expectation to within roughly two percent for every seed.

\section{Initialization and Input-Decoding Controls}
\label{app:controls}

We re-run the training script's deterministic initialization procedure, replay the saved validation inputs, and repeat the $16$-input intervention. No initialization checkpoints were saved, so this reproduces the initialization procedure rather than stored pre-training weights.

For input decoding, four-fold sequence-grouped ridge regression with regularization $10^{-4}$ predicts the four input bits from either $D(s,u)$ or $r(s,u)=D(s,u)-\mu(s)$. All interventions from a source sequence remain in one fold, and vector accuracy requires all four thresholded bits to be correct.

\begin{table}[ht]
\caption{Initialization and input-decoding controls. Mean and standard deviation over five seeds where applicable. Residual decoding is $1.000$ in every seed.}
\label{tab:controls}
\centering
\begin{tabular}{lcccc}
\toprule
architecture & initialized $F$ & trained $F$ & full-$D$ decode & residual decode \\
\midrule
Elman & $0.657\pm0.019$ & $0.549\pm0.015$ & $1.000\pm0.000$ & $1.000$ \\
GRU & $0.265\pm0.010$ & $0.947\pm0.016$ & $0.995\pm0.006$ & $1.000$ \\
LSTM & $0.274\pm0.009$ & $0.938\pm0.021$ & $0.804\pm0.013$ & $1.000$ \\
\bottomrule
\end{tabular}
\end{table}

\section{Estimator Details}
\label{app:estimator}

The original dense-energy audit uses $S=h$ for all architectures. The matched held-out mean comparison uses $S=h$ for Elman/GRU and $S=(h,c)$ for LSTM. Gaussian weights depend on this conditioning state through
\begin{equation}
w_{qi}=\frac{\exp[-\|S_q-S_i\|^2/(4\varepsilon)]}
{\sum_j\exp[-\|S_q-S_j\|^2/(4\varepsilon)]},
\end{equation}
with reference bandwidth $\varepsilon_{\mathrm{med}}=\tfrac14\operatorname{median}_{i<j}\|S_i-S_j\|^2$. We evaluate scales $0.02$, $0.05$, $0.1$, $0.2$, $0.5$, and $1.0$ on internal write transitions. For full-state LSTM kernels, each block is divided by its RMS radius estimated on training sequences before concatenation. The dense audit includes each query's own transition. Held-out predictions exclude all transitions from query sequences.

Four outer sequence-grouped folds evaluate predictions. Three inner folds select each smoother's bandwidth by observed-displacement prediction score. Enumerated means are used only for evaluation. Training folds determine feature scaling and median distances. Ridge uses an intercept and fixed regularization $10^{-4}$. The score is
\begin{equation}
R_{\mathrm{pred}}=1-\frac{\sum_i\|D_i-\widehat\mu_i\|^2}{\sum_i\|D_i\|^2}.
\end{equation}
For an independently evaluated population predictor $f(S)$, this score equals $F(S)-\mathbb E\|f(S)-\mathbb E[D\mid S]\|^2/\mathbb E\|D\|^2$, so it is at most the matching population fraction. The identity also holds on the sampled queries when averaging over all $16$ inputs. Scores on one observed input per query can fluctuate around that reference.

\paragraph{Direct conditional-mean error.}
Table~\ref{tab:direct-error} reports mean Euclidean error and normalized squared error,
\begin{equation}
E_\mu=\frac{\sum_q\|\widehat\mu_q-\mu(S_q)\|^2}{\sum_q\mathbb E_u\|D(S_q,u)\|^2}.
\end{equation}
Errors pool outer-fold predictions within each seed. The original nonnested $h$-only sweep is retained solely for the descriptive bandwidth/energy comparison in Section~\ref{sec:estimation}. Figure~\ref{fig:exact}C uses the matched nested evaluation.

\begin{table}[ht]
\caption{Direct held-out conditional-mean errors, averaged over five seeds. The $\pm$ values denote sample standard deviations. Mean norm error uses tuned/fixed predictions, in hidden RMS units. Median uses scale $1$. Tuned kernels use nested selection, while ridge has fixed regularization. LSTM conditions on $(h,c)$.}
\label{tab:direct-error}
\centering
\small
\begin{tabular}{llrrr}
\toprule
architecture & estimator & mean $\|\widehat\mu-\mu\|$ & median $E_\mu$ & tuned/fixed $E_\mu$ \\
\midrule
Elman & successor & 0.224 & 0.050 & $0.039\pm0.004$ \\
 & displacement & 0.308 & 0.337 & $0.074\pm0.007$ \\
 & ridge & 0.168 & --- & $0.025\pm0.001$ \\
GRU & successor & 0.167 & 0.745 & $0.137\pm0.009$ \\
 & displacement & 0.168 & 0.559 & $0.142\pm0.029$ \\
 & ridge & 0.066 & --- & $0.026\pm0.015$ \\
LSTM & successor & 0.138 & 0.671 & $0.077\pm0.025$ \\
 & displacement & 0.160 & 0.564 & $0.102\pm0.033$ \\
 & ridge & 0.065 & --- & $0.018\pm0.005$ \\
\bottomrule
\end{tabular}
\end{table}

The accompanying analysis files retain per-seed errors, bandwidth choices, the $h$-only LSTM control, and replay checks. The full-state kernel's performance depends on its metric and finite-sample neighborhoods. Extra conditioning information raises the population ceiling without guaranteeing a better predictor.

\clearpage

\end{document}